\documentclass[10pt,journal,compsoc,twoside]{IEEEtran}
\ifCLASSOPTIONcompsoc
  \usepackage[nocompress]{cite}
\else
  \usepackage{cite}
\fi

\ifCLASSINFOpdf
\else
\fi

\usepackage{bm}
\usepackage{amssymb}
\usepackage{amsthm}
\usepackage{amsmath}
\usepackage{url}
\usepackage{graphicx}
\usepackage{xcolor}
\usepackage{stfloats}
\usepackage{algorithm}
\usepackage{algorithmic}

\newtheorem{definition}{Definition}
\newtheorem{proposition}{Proposition}

\newcommand{\R}{\mathbb{R}}
\makeatletter
\newcommand*{\defeq}{\mathrel{\rlap{%
			\raisebox{0.3ex}{$\m@th\cdot$}}%
		\raisebox{-0.3ex}{$\m@th\cdot$}}%
	=}

\DeclareMathSymbol{\shortminus}{\mathbin}{AMSa}{"39}

\fnbelowfloat

\definecolor{gray}{rgb}{0.3,0.3,0.3}

\begin{document}

\title{Cluster Purging: Efficient Outlier Detection based on Rate-Distortion Theory}

\author{Maximilian~B.~Toller, 
        Bernhard~C.~Geiger,~\IEEEmembership{Senior Member,~IEEE,}
        and~Roman~Kern 
\IEEEcompsocitemizethanks{\IEEEcompsocthanksitem M.B. Toller and Bernhard C. Geiger are with Know Center GmbH, Graz, Austria;
E-mail: \{mtoller, bgeiger\}@know-center.at

\IEEEcompsocthanksitem Roman Kern is with Graz University of Technology, Graz, Austria;\protect\\ E-mail: rkern@tugraz.at}%
\thanks{Manuscript received 7 Oct. 2020; revised 2 Jun. 2021; accepted 27 July 2021. Date of publication 10 Aug. 2021; Date of current version: 10. Jan 2023.\\ (Corresponding author: Maximilian B. Toller.)\\ Recommended for acceptance by P. Bogdanov\\ Digital Object Identifier no. 10.1109/TKDE.2021.3103571}%
}

\markboth{IEEE TRANSACTIONS ON KNOWLEDGE AND DATA ENGINEERING,~Vol.~35, No.~2, February~2023}{Toller \MakeLowercase{\textit{et al.}}: Cluster Purging: Efficient Outlier Detection based on Rate-Distortion Theory}

\IEEEtitleabstractindextext{%
\begin{abstract}

Rate-distortion theory-based outlier detection builds upon the rationale that a good data compression will encode outliers with unique symbols.
Based on this rationale, we propose Cluster Purging, which is an extension of clustering-based outlier detection.
This extension allows one to assess the representivity of clusterings, and to find data that are best represented by individual unique clusters.
We propose two efficient algorithms for performing Cluster Purging, one being parameter-free, while the other algorithm has a parameter that controls representivity estimations, allowing it to be tuned in supervised setups.
In an experimental evaluation, we show that Cluster Purging improves upon outliers detected from raw clusterings, and that Cluster Purging competes strongly against state-of-the-art alternatives.
\end{abstract}

}

\maketitle

\IEEEdisplaynontitleabstractindextext

\IEEEpeerreviewmaketitle

\IEEEraisesectionheading{\section{Introduction}\label{sec:introduction}}
\IEEEPARstart{I}{n} present days, there exists an abundance of datasets containing individual observations that greatly deviate from the remaining observations, commonly called \textit{outliers} or \textit{anomalies}.
The task of finding such outlying/anomalous observations in datasets is relevant in a multitude of applications and has received much attention in the last decades~\cite{chandola2009anomaly}.
Traditionally, outlier detection was mostly approached from a statistical perspective, where data are modeled with distributions, while recently database-oriented methods that focus on efficiency and scalability have become more popular~\cite{zimek2018there}.
A major part of contemporary research concentrates on using  deep learning to detect outliers in semi-supervised~\cite{gornitz2013toward,pang2019deep} or unsupervised~ \cite{chen2017outlier,zenati2018adversarially,ruff2018deep} settings.
These approaches are well motivated for high-dimensional datasets and have yielded significantly improved outlier detection accuracy on benchmark datasets~\cite{lecun2015deep,kwon2017survey,chalapathy2019deep}, yet deep learning techniques are also criticized for being data hungry~\cite{marcus2018deep} and lacking interpretability~\cite{rudin2019stop}.
Both of these deficits gravely affect outlier detection since in many research fields large training datasets are not available~\cite{pang2019deep}.
Further, outlier detection techniques are commonly used in high-risk applications such as intrusion detection~\cite{chandola2009anomaly}, where black-box models should generally be avoided~\cite{rudin2018optimized}.

In contrast, \textit{clustering-based} outlier detection methods~\cite{chandola2009anomaly} resort to very intuitive concepts of what an outlier might possibly be; for instance observations that have abnormal local density~\cite{ester1996density}; or observations that do not fit well into any cluster~\cite{chawla2013k,gan2017k,liu2019clustering}.
A trait that these methods have in common is that they detect outliers during clustering, for instance by assigning outliers to a special outlier cluster.
While this trait can be advantageous in several settings, it also has the downside that outliers are only detected as a ``side-product'' of clustering~\cite{chandola2009anomaly}.
As a consequence, outliers detected by methods such as \cite{ester1996density,chawla2013k,gan2017k,liu2019clustering} are observations that are irregular in the respective clustering, yet not necessarily irregular with respect to the (unclustered) data.

Another type of clustering-based methods infers outliers after the raw data were clustered.
For instance, the Cluster-Based Local Outlier Factor (CBLOF)~\cite{he2003discovering} scales distances between observations and cluster centers by cluster sizes, regardless of which clustering technique was used.
Hence, CBLOF allows one to choose a clustering method that is well-suited for the data at hand.
However, outlier detection techniques such as CBLOF~\cite{he2003discovering,jiang2008clustering,pamula2011outlier} still have the same drawback as the methods mentioned above:
They assume that the computed clustering is sufficient for describing outliers in raw data, which can be problematic in scenarios where it is challenging to perform a good clustering, e.g. in high-dimensional data~\cite{kriegel2009clustering}.

To address this issue, one may resort to information theory.
From an information-theoretic perspective, a clustering is a lossy compression of the raw data~\cite{dhillon2003information}, where a raw observation is represented by the cluster it was assigned to.
The loss (distortion) that occurs during such a clustering-compression can be combined with a cluster's degree of compression (rate) to quantify how well this cluster represents the observations that are assigned to it.
Further, rate-distortion theory allows one to infer how the representivity of a clustering would change if one were to modify this clustering, and which observations would be better represented by different clusters (cf. \cite{blahut1972computation,arimoto1972algorithm}).
Observations that are hard to represent by a meaningful cluster and that are best represented by themselves can then be considered as outliers.

This description outlines a technique that we refer to as \textit{Cluster Purging}, in analogy to the act of purging in authoritarian political systems where deviating individuals that are not well-represented by such systems are removed from society\footnote{None of the authors or their affiliations approve of political purges in any form.}.
In short, Cluster Purging is performed by modifying a clustering (or by analyzing a set of given clusterings), and then isolating observations that are not represented well by their cluster, regardless of how one modifies it (or which of the clusterings one considers).
As such, Cluster Purging is, to the best of our knowledge, a conceptually novel approach to cluster-based outlier detection, and the main contributions of this work stem from it:
\begin{itemize}
	\item Review of related work, outlining the differences between Cluster Purging and existing methods \linebreak(Section~\ref{sec:related_work}).
	\item Theoretical formalization of Cluster Purging and description of required concepts from information theory (Section~\ref{sec:theory}).
	\item Description of a parameter-free algorithm for Cluster Purging and discussion of various aspects that are relevant in practice, i.e. efficiency, interpretation of proposed outliers, how one can introduce parameters for improved performance, and limitations \linebreak(Section~\ref{sec:practice}).
	\item Empirical demonstration that Cluster Purging improves upon outliers detected from clustering alone, and that Cluster Purging strongly competes against state-of-the-art alternatives (Section~\ref{sec:experiments}).
\end{itemize}

\section{Related Work}\label{sec:related_work}
In general, cluster-based outlier detection techniques can be split into three categories depending on how they define outliers~\cite{chandola2009anomaly}:
\begin{enumerate}
	\item Outliers are observations that do not fit into any cluster.\label{cat:1}
	\item Outliers are far away from their cluster's centroid.\label{cat:2}
	\item Outliers are assigned to small or sparse clusters.
\end{enumerate}
Conceptually, category~\ref{cat:1} is most closely related to Cluster Purging, since in our method outliers are observations that cannot be represented well by any cluster.
There are several existing methods that fall into category~\ref{cat:1}, for instance Density-Based Spatial Clustering of Applications with Noise (DBSCAN)~\cite{ester1996density}, extensions of DBSCAN such as~\cite{ruiz2007c,smiti2013soft}, k-Means$\shortminus\shortminus$~\cite{chawla2013k} and k-Means with Outlier Removal~\cite{gan2017k}.
However, a key difference between these methods and Cluster Purging is that our method is not bound to a specific clustering.
Even if one bases Cluster Purging on one of the above clusterings, the results can be very different since our method does not assume that a single clustering necessarily describes outliers in the raw data.

Surprisingly, one can argue that our method should also fall into category~\ref{cat:2}, since the theoretical formulation of Cluster Purging permits setups where outliers are observations that are far away from a centroid (see Section~\ref{sec:theory}).
Related methods from this category are techniques that combine centroid-based clusterings with a distance threshold, for instance~\cite{he2002outlier,pamula2011outlier}.
One can distinguish Cluster Purging from these methods by the simple fact that our method does not require a distance threshold (although Cluster Purging can be adapted to require one, should an application demand this (see Section~\ref{sec:practice})).

Typical methods of the third category are Local Outlier Factor~\cite{breunig2000lof} and its numerous variants, e.g.~\cite{papadimitriou2003loci,duan2007local,kriegel2009loop}.
The Cluster-Based Local Outlier Factor (CBLOF)~\cite{he2003discovering} is particularly noteworthy, since this method is directly applicable to any clustering, similar to Cluster Purging.
The main difference between CBLOF and Cluster Purging is that, while our method can be based on local densities, it does not require a threshold parameter to infer critical differences in local densities and does not consider a single clustering as sufficient for describing outliers.

From a theoretical perspective, the most closely related method to ours is the one-class rate-distortion model (OCRD)~\cite{crammer2008rate}.
The brief description of Cluster Purging given above can be seen as a single (half-)step of the Blahut-Arimoto algorithm~\cite{blahut1972computation,arimoto1972algorithm,cover2006elements}, which OCRD adapts for one-class classification.
However, while OCRD is optimal in a rate-distortion theoretic sense, we here do not aim for this optimality.
Instead, Cluster Purging supports arbitrary clustering techniques, allowing for a greater flexibility.
In our experiments, we demonstrate that rate-distortion optimal clusterings are not necessarily optimal for detecting outliers in real data (Section~\ref{sec:experiments}).

\section{Theoretical Formulation}\label{sec:theory}
In this section, the theoretical background of Cluster Purging is explained and the concept of representivity is introduced.
In short, clustering can be interpreted as a form of data compression that yields cluster assignments and a representation.
One can measure how representative such a representation is via its surplus complexity when compared to the most representative clustering at a given inaccuracy.
Since directly finding the most representative clustering is often infeasible, we show how representivity can be efficiently estimated from a small set of available clusterings.
Finally, we show how one can detect outliers under the premise that a good clustering would represent outliers by themselves, i.e. with an additional cluster.
\subsection{Background}
\subsubsection{Data Compression}
Let $\bm{x} = \{x_1,\mathellipsis,x_n\}$ be a dataset of $n$ observations in $\R^d$ consisting of $u\approx n$ unique values.
A common data analysis goal is to obtain a representation of $\bm{x}$ that has fewer unique values without losing too much information~\cite{yi2000fast,lin2003symbolic,keogh2001dimensionality}.
In coding theory, the task of finding such a representation  consisting of $\nu \ll u$ unique symbols is referred to as lossy data compression.
Clustering can be seen as a typical example for lossy data compression.
In detail, a successful compression via (non-fuzzy) clustering yields two objects\begin{enumerate}
	\item A list of $n$ cluster assignments $\bm{c} = (c_1,\mathellipsis,c_n)$, where $c_j\in1,\mathellipsis,\nu$  is the index of the cluster that contains observation $x_j$.
	\item A low-dimensional representation $\bm{r}=(r_1,\mathellipsis,r_\nu)$ describing $\nu$ different clusters.
\end{enumerate}
A visualization can be seen in Fig.~\ref{fig:cluster_to_symbol}.
\begin{figure}[!t]
	\centering
	\includegraphics[scale=0.75]{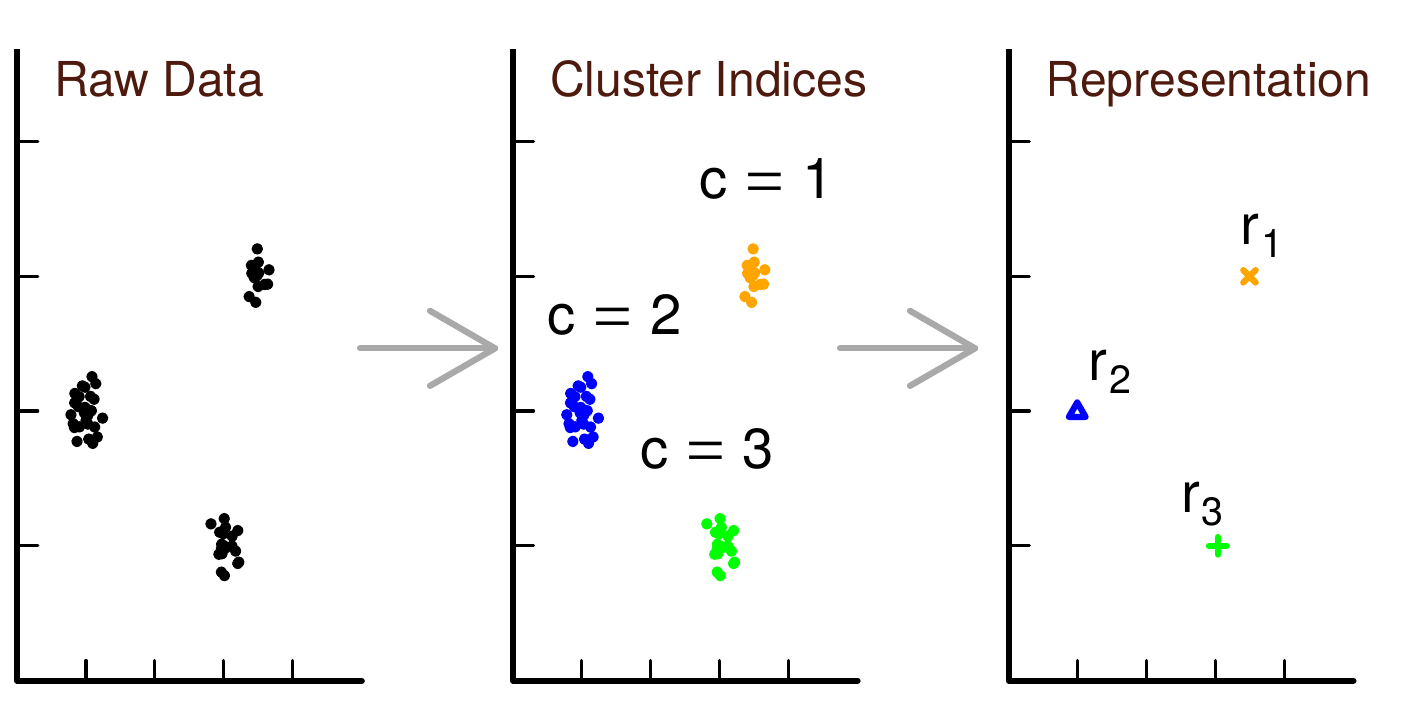}
	\caption{Compression via $k$-means clustering. \textit{Left}: A dataset consisting of 65 observations. \textit{Middle}: Cluster assignments, indicated by color. \textit{Right}: Symbols representing each cluster.}
	\label{fig:cluster_to_symbol}
\end{figure}
Not all clustering techniques return both of these objects, e.g. DBSCAN only gives cluster assignments $\bm{c}$ yet no representation $\bm{r}$.
Details on how to obtain representations in such cases are given in Section~\ref{sec:practice_perturbation}.

Further, assume that a small subset of outliers $\bm{x_y} = \{x_{y_1},\mathellipsis,x_{y_m}\}$ with $ m \ll n$ is part of the dataset.
Since outliers are commonly assumed to deviate significantly from the remaining observations~\cite{grubbs1969procedures}, compressing a dataset that contains outliers will either require additional unique symbols for outliers or else lead to a less effective compression~\cite{bohm2009coco}.
Let
\begin{equation}
	d(\bm{x},\bm{r})=\sum_{j=1}^{n}d(x_j,r_{c_j})\label{eq:distortion}
\end{equation} be a separable distortion function, i.e. a measure describing how accurately $\bm{r}$ represents dataset $\bm{x}$.
If an outlier is represented by the same symbol as an inlier, then this will increase the overall distortion since inliers and outliers are assumed to be dissimilar.
Consequently, one can reduce the overall distortion by compressing outliers to unique symbols.
In the context of clusterings, this translates to assigning outliers to singleton clusters, i.e. an additional cluster that only contains $x_{y_j}$.
However, adding unique outlier clusters also increases the overall \textit{complexity} of the compression.

\subsubsection{The Empirical Rate-Distortion Function}

Rate-distortion theory seeks to describe this trade-off between representation complexity (\textit{rate}) and inaccuracy (\textit{distortion}) in the context of random variables.
Formally, the rate-distortion function $R(D)$ of a random variable $X$ is defined as (cf.~\cite{cover2006elements})
\begin{equation}\label{eq:RD}
R(D) = \min_{P(\hat{X}|X)} H(\hat{X}) - H(\hat{X}|X)\; \textbf{subject to} \; d(X,\hat{X}) \le D
\end{equation}
where $P(\cdot)$ and $H(\cdot)$ are the probability and entropy functions, respectively, $\hat{X}$ is a stochastic compression of $X$, and $D$ is a specific distortion value, e.g. the sum of squared errors in a $k$-means clustering.
Intuitively, the rate-distortion function describes the smallest complexity one can achieve while compressing $X$ at a given distortion, regardless of how the compression is performed.

To transfer this stochastic definition into a real-data context, let 
\begin{equation}\label{eq:entropy}
	h(\bm{c}) = -\sum_{f\in\bm{f}^{\bm{c}}} \frac{f}{n}\log\frac{f}{n}.
\end{equation}
be the empirical counterpart to the theoretical entropy $H(\hat{X})$ as per \cite{cover2006elements}, where $\bm{f}^{\bm{c}}=\{f^{\bm{c}}_{1},\mathellipsis,f^{\bm{c}}_{\nu}\}$ are the numbers of observations assigned to each cluster.
Then, inspired by~\eqref{eq:RD}, we define the empirical rate-distortion function of a dataset $\bm{x}$ as
\begin{equation}\label{eq:RD_ds}
R(D,\bm{x},C) \defeq \min_{\{C(\bm{x},\bm{\theta}):\bm{\theta}\in\bm{\Theta}\}} h(\bm{c}) \;\; \textbf{subject to} \; \; d(\bm{x},\bm{r}) \le D
\end{equation}
with $C(\bm{x},\bm{\theta})=(\bm{c},\bm{r})$, where $C(\cdot)$ is a deterministic compression function (i.e. a non-fuzzy clustering technique) and $\bm{\theta}$ are its parameters and where $\bm{\Theta}$ is the set of all possible parametrizations.
Intuitively, the empirical rate-distortion function can be seen as the strongest degree of compression one can achieve on a dataset with a fixed compression method without exceeding the required distortion.
\begin{figure}[!t]
	\includegraphics[scale=0.75]{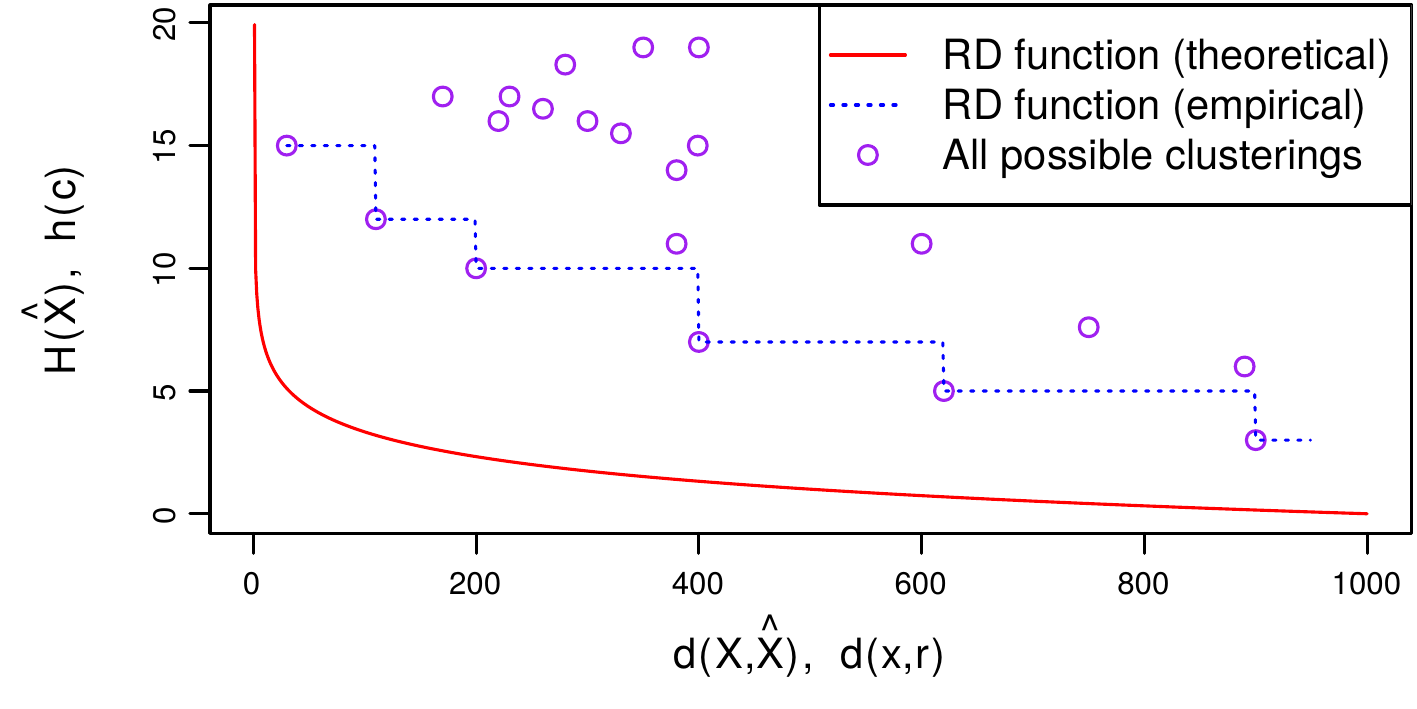}
	\caption{Comparison of theoretical and empirical rate-distortion functions.}
	\label{fig:rd_funcs}
\end{figure}
As such, it describes the trade-off between compression complexity and inaccuracy for a fixed dataset and a specific clustering method.
The term $h(\bm{c}|\bm{x})$ was omitted from \eqref{eq:RD_ds}, since $h(\bm{c}|\bm{x})=0$ for all non-fuzzy clustering techniques.
A visualization of theoretical and empirical rate-distortion functions is depicted in Fig.~\ref{fig:rd_funcs}.

\subsection{Measuring Cluster Representivity}
\subsubsection{Theoretical Representivity}
From a rate-distortion theoretical perspective, there are two quantities that measure how ``good'' a clustering $(\bm{c},\bm{r})$ represents the raw data
\begin{enumerate}
	\item The degree of compression (the rate), computed via entropy $h(\bm{c})$;
	\item How accurate the representation is (the distortion), computed via distortion $d(\bm{x},\bm{r})$.
\end{enumerate}
While the empirical rate-distortion function $R(D,\bm{x},C)$ describes the best achievable trade-off between these quantities in a given setup, the average result of a clustering algorithm typically offers a worse trade-off.
More concretely, for every clustering $C(\bm{x},\bm{\theta})=(\bm{c},\bm{r})$ it holds that
\begin{equation}\label{eq:ineq}
	R(d(\bm{x},\bm{r}),\bm{x},C) \le h(\bm{c})
\end{equation}
since the rate-distortion function describes the global minimum over all parametrizations, i.e the best achievable representation at distortion $d(\bm{x},\bm{r})$.
Due to this inequality there is always a nonnegative surplus complexity between $(\bm{c},\bm{r})$ and \eqref{eq:RD_ds}.
Thus, one can measure the theoretical representivity of a clustering via \begin{equation}
	\rho(\bm{x},\bm{c},\bm{r},C) \defeq  R(d(\bm{x},\bm{r}),\bm{x},C) \:/\: h(\bm{c}).
\end{equation}
However, computing $R(d(\bm{x},\bm{r}),\bm{x},C)$ and thus $\rho(\bm{x},\bm{c},\bm{r},C)$ is infeasible for many clustering techniques, since this would require one to compute $C(\bm{x},\bm{\theta})$ for all possible clustering parameters $\bm{\theta}$.
Therefore, it is more practical to estimate clustering representivity relative to a small set of representations, obtained from parametrizations $\{\bm{\theta}_1,\mathellipsis,\bm{\theta}_t\}$. 
We refer to this estimate as rate-distortion \textit{hull}.

\begin{definition}\label{def:rd_hull}
	\textit{Rate-distortion hull}.
	Let $\underline{\bm{c}} = (\bm{c}_1,\mathellipsis,\bm{c}_t)$ and $\underline{\bm{r}} = (\bm{r}_1,\mathellipsis,\bm{r}_t)$ be a set of clustering assignments and representations, respectively, obtained by evaluating clustering technique $C(\cdot)$ on dataset $\bm{x}$ with parametrizations $\{\bm{\theta}_1,\mathellipsis,\bm{\theta}_t\}$.
	Further, let $\bm{v} = [v_1,\mathellipsis,v_s]$ be the indices of the lower convex hull of the arising distortion-entropy pairs $\{[d(\bm{x},\bm{r}_1),h(\bm{c}_1)],\mathellipsis,[d(\bm{x},\bm{r}_t),h(\bm{c}_t)]\}$.
	Then, the rate-distortion hull of $\underline{\bm{c}}$ and $\underline{\bm{r}}$ is given by
	\begin{align}
	&\mathcal{L}(D,\underline{\bm{c}},\underline{\bm{r}}) \defeq \kappa_i D + \delta_i \quad \forall i \in \{2,\mathellipsis,s\} \label{eq:rd_curve}\\ &D \in [d(\bm{x},\bm{r}_{v_{1}}),d(\bm{x},\bm{r}_{v_s})]  \nonumber
	\end{align}
	where 
	\begin{equation}\label{eq:k}
	\kappa_i = \frac{h(\bm{c}_{v_i})-h(\bm{c}_{v_{i-1}})}{d(\bm{x},\bm{r}_{v_i})-d(\bm{x},\bm{r}_{v_{i-1}})} \end{equation}
	and
	\begin{equation}\label{eq:delta}
	 \delta_i = h(\bm{c}_{v_{i}}) -\kappa_i\cdot d(\bm{x},\bm{r}_{v_i}) \end{equation} are the slopes and vertical intercepts of the arising linear pieces, with $d(\bm{x},\bm{r}_{v_1}) < \cdots < d(\bm{x},\bm{r}_{v_s})$.
\end{definition}
Intuitively, a rate-distortion hull is a linear interpolation of the lower convex hull of the entropy and distortion values associated with observed clusterings $(\underline{\bm{c}},\underline{\bm{r}})$.
A visualization of a rate-distortion hull is shown in Fig.~\ref{fig:rd_convex}.

Further, since $\mathcal{L}(\cdot,\underline{\bm{c}},\underline{\bm{r}}) = \mathcal{L}(\cdot,\underline{\bm{c}}_{\bm{v}},\underline{\bm{r}}_{\bm{v}})$, we assume without loss of generality that $v_i=i$ and $s=t$ to keep the notation simple.

\begin{figure}[!t]
	\includegraphics[scale=0.75]{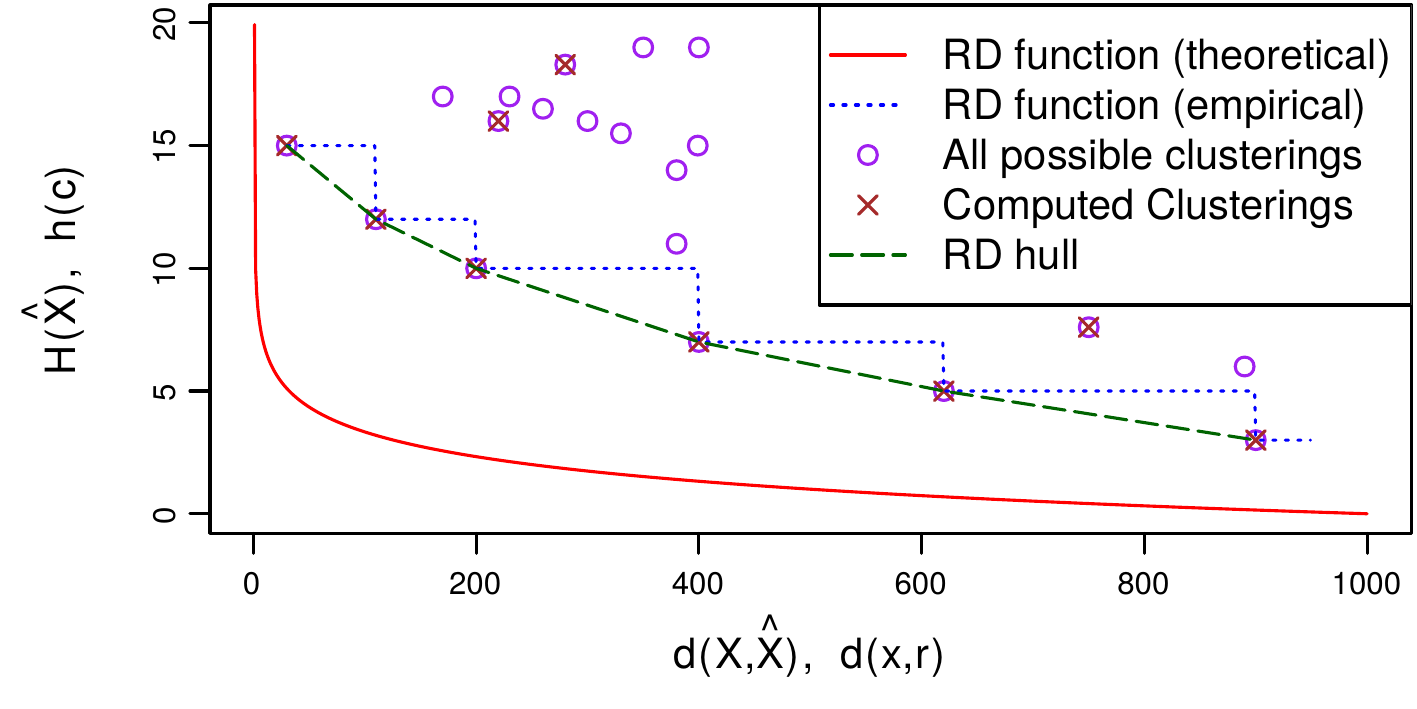}
	\caption{Comparison of theoretical rate-distortion function, empirical rate-distortion function and rate-distortion hull. If ideal clusterings are selected for estimating the empirical rate-distortion function, then the resulting rate-distortion hull is equal to the lower convex hull of the empirical rate-distortion function.}
	\label{fig:rd_convex}
\end{figure}

\subsubsection{Representivity after Modification}
Naturally, it is not possible to directly estimate the theoretical representivity of clusterings $(\underline{\bm{c}},\underline{\bm{r}})$ based on a rate-distortion hull $\mathcal{L}(\cdot,\underline{\bm{c}},\underline{\bm{r}})$ constructed from the same clusterings.
However, one can use $\mathcal{L}(\cdot,\underline{\bm{c}},\underline{\bm{r}})$ for estimating how the representivity of a particular clustering $(\bm{c}_i,\bm{r}_{i})\in(\underline{\bm{c}},\underline{\bm{r}})$ reacts to arbitrary modifications via

\begin{equation}
	\hat{\rho}(\bm{x},\bm{c}_{i}',\bm{r}_{i}',\underline{\bm{c}},\underline{\bm{r}}) \defeq \mathcal{L}(d(\bm{x},\bm{r}_{i}'),\underline{\bm{c}},\underline{\bm{r}}) \:/\: h(\bm{c}_{i}') 
\end{equation}
where $\bm{c}_i'$ and $\bm{r}_{i}'$ are arbitrarily modified versions of $\bm{c}_i$ and $\bm{r}_{i}$ respectively, with \[\bm{c}_{i}'\notin\underline{\bm{c}} \quad\text{and}\quad\bm{r}_{i}'\notin\underline{\bm{r}}.\]
Note that the error between measurements $\hat{\rho}(\bm{x},\bm{c}_i',\bm{r}_{i}',\underline{\bm{c}},\underline{\bm{r}})$ and $\rho(\bm{x},\bm{c}_i',\bm{r}_{i}',C)$ will not only depend on the clusterings used for constructing the rate-distortion hull.
It will also depend on how many $c\in\bm{c}_i$ and $r \in\bm{r}_i$ were modified.
Generally speaking, the more similar modified clustering $(\bm{c}_i',\bm{r}_{i}')$ is to $(\bm{c}_i,\bm{r}_{i})$, the smaller the error between $\hat{\rho}(\bm{x},\bm{c}_i',\bm{r}_{i}',\underline{\bm{c}},\underline{\bm{r}})$ and $\rho(\bm{x},\bm{c}_i',\bm{r}_{i}',C)$ will be.

\subsection{Detecting Outliers with Cluster Representivity}
\subsubsection{Definition of Rate-Distortion Outliers}
Since $\hat{\rho}(\bm{x},\cdot,\cdot,\underline{\bm{c}},\underline{\bm{r}})$ allows one to measure the effect of arbitrary modifications to a clustering, one can also measure how assigning an individual observation to a new, unique cluster would affect representivity.
Now recall from above that an outlier is an observation that will likely need a unique symbol for an effective compression~\cite{bohm2009coco}.
If changing the cluster assignment of observation $x_j$ in $\bm{c}_i$ to a new additional cluster would improve $\bm{r}_i$'s representivity, then $x_j$ should be labeled as outlier.
This intuition can be formalized as follows.

\begin{definition}\label{def:rd_outlier}
	Rate-distortion outlier.
	Let $\bm{x}$ be a dataset and $(\underline{\bm{c}},\underline{\bm{r}})$ a set of clusterings.
	Then observation $x_j$ is a rate-distortion outlier if
	\begin{equation}
		\hat{\rho}\left(\bm{x},\bm{c}'_{(i,j)},\bm{r}'_{(i,j)},\underline{\bm{c}},\underline{\bm{r}}\right) \ge 1 \quad \forall i\in[2\mathellipsis,t]
	\end{equation}
	with
	\begin{equation}\label{eq:cdash}
		\bm{c}'_{(i,j)} = (c_{i,1},\mathellipsis,c_{i,j\shortminus1},\nu+1,c_{i,j+1},\mathellipsis,c_{i,n})
	\end{equation}
	and
	\begin{equation}\label{eq:rdash}
		\bm{r}'_{(i,j)} = (r_{i,1},\mathellipsis,r_{i,\nu},r^\star)
	\end{equation}
	where $r^\star$ is a representation of $x_j$ such that $d(x_j,r^\star)=0$.
		
\end{definition}
In simple terms, Definition~\ref{def:rd_outlier} states that $x_j$ is a rate-distortion outlier if assigning it to $r^\star$ would improve the representivity of all clusterings $(\underline{\bm{c}},\underline{\bm{r}})$.

\subsubsection{Computation of $\hat{\rho}(\bm{x},\bm{c}'_{(i,j)},\bm{r}'_{(i,j)},\underline{\bm{c}},\underline{\bm{r}})$}\label{sec:ed_computation}
A key advantage of defining outliers as in Definition~\ref{def:rd_outlier} is that  $\hat{\rho}(\bm{x},\cdot,\cdot,\underline{\bm{c}},\underline{\bm{r}})$ can be computed for $\bm{c}'_{(i,j)}$ and $\bm{r}'_{(i,j)}$ from a set of clusterings $(\underline{\bm{c}},\underline{\bm{r}})$ in $\mathcal{O}(n)$ time.
This works, since the change in entropy from $\bm{c}_i$ to $\bm{c}'_{(i,j)}$ and the change in distortion from $\bm{r}_i$ to $\bm{r}'_{(i,j)}$ can be computed independently from the remaining clusterings in $(\underline{\bm{c}},\underline{\bm{r}})$.
\begin{proposition}
	Let $\bm{c}$ be a list of cluster assignments and let  $\bm{f}^{\bm{c}}=\{f^{\bm{c}}_{1},\mathellipsis,f^{\bm{c}}_{\nu}\}$ be the numbers of observations assigned to each cluster.
	Then the change in entropy  caused by assigning $x_j$ to an additional unique cluster, yielding $\bm{c}'$, depends only on $f^{\bm{c}}_{c_j}$ and is given by 
	\begin{equation}\label{eq:change_e}
		h(\bm{c}') \shortminus h(\bm{c}) = \frac{1}{n} \left(f^{\bm{c}}_{c_j} \log f^{\bm{c}}_{c_j} \shortminus (f^{\bm{c}}_{c_j}\shortminus 1)\log(f^{\bm{c}}_{c_j}\shortminus 1)\right).
	\end{equation}
\end{proposition}

\begin{proof}
	The entropy of $\bm{c}$ as given in~\eqref{eq:entropy} can be rewritten as 
	\begin{equation}\label{eq:h_r}
		h(\bm{c})=\log n - \frac{1}{n}\sum_{f\in\bm{f}^{\bm{c}}} f \log f
	\end{equation}
	since $\log\frac{f}{n} = \log f - \log n$.
	The entropy of $\bm{c}'$ is given by 
	\begin{equation}\label{eq:h_r'}
		h(\bm{c}') = \log n - \frac{1}{n}\sum_{f\neq f^{\bm{c}}_{c_j}} f \log f \;\; -  \frac{1}{n}\left((f^{\bm{c}}_{c_j}\shortminus 1) \log (f^{\bm{c}}_{c_j} \shortminus 1)\right)
	\end{equation}
	since $1$ observation is removed from cluster $c_j$ and a unique cluster is added with entropy $1 \log 1 = 0$.
	Subtracting \eqref{eq:h_r} from \eqref{eq:h_r'} yields \eqref{eq:change_e}.
\end{proof}

\begin{figure*}[!t]
	\centering
	\includegraphics[scale=0.725]{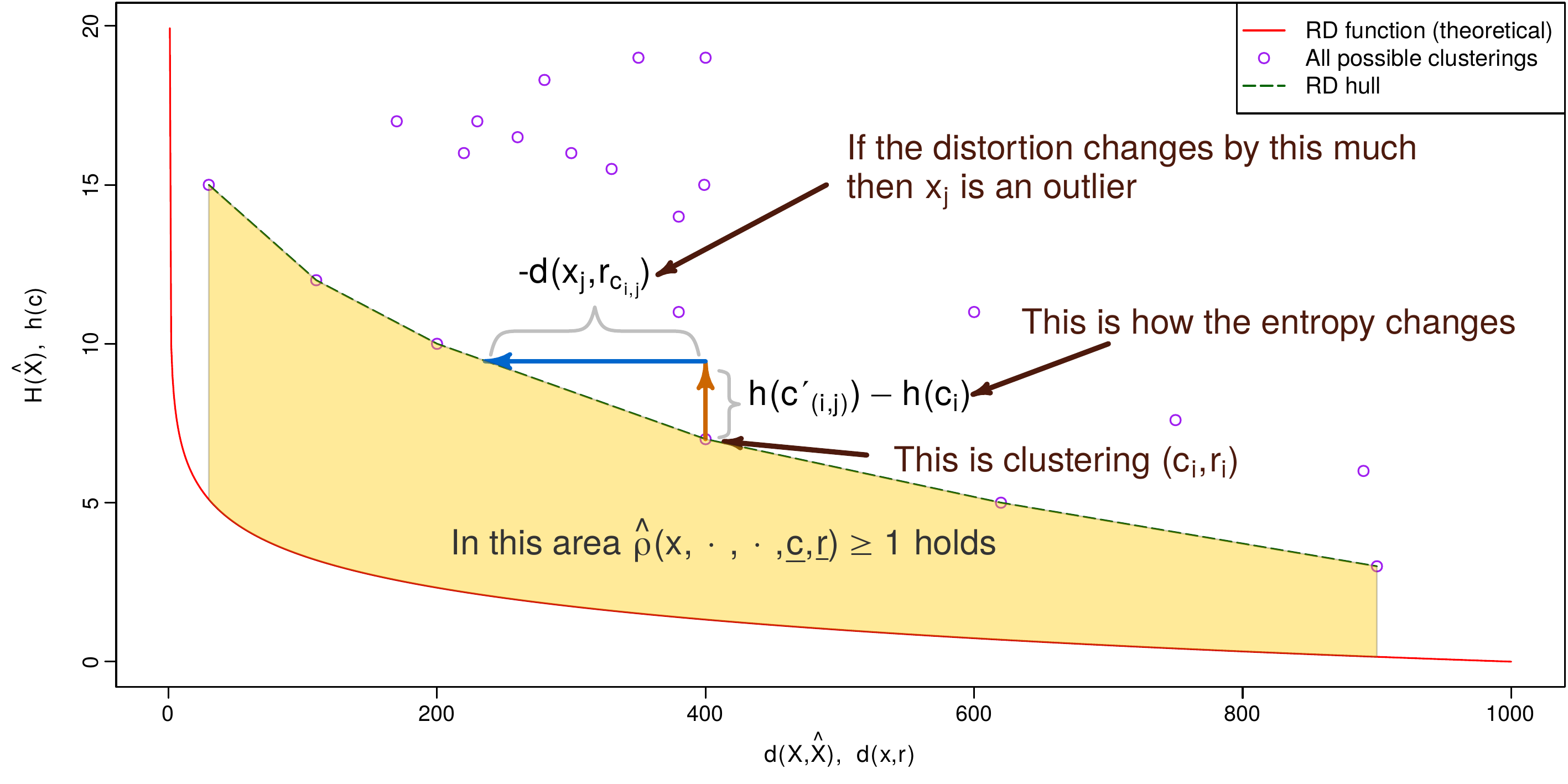}
	\caption{Geometric interpretation of the computation of cluster representivity. For every clustering on the rate-distortion hull, one can compute how the entropy would change if $x_j$ were represented by a new cluster. If this new clustering has a distortion that is sufficiently small to enter the area beneath the rate-distortion hull, then $x_j$ is an outlier that needs to be represented by itself rather than a cluster.}
	\label{fig:rho}
\end{figure*}

The change in distortion from $\bm{r}_i$ to $\bm{r}'_{(i,j)}$ is given by
\begin{equation}\label{eq:change_d}
	d(\bm{x},\bm{r}'_{(i,j)}) - d(\bm{x},\bm{r}_i) = -d(x_j,r_{c_{i,j}})
\end{equation}
which follows by assumption from Definition~\ref{def:rd_outlier}.
Intuitively, when one assigns $x_j$ to a new unique symbol, then this symbol perfectly represents $x_j$ and hence the total distortion decreases by $d(x_j,\bm{r}_{i,j})$.
Note that \eqref{eq:change_d} only depends on observation $x_j$ and the cluster representative $x_j$ is assigned to, i.e. $r_{c_{i,j}}$.

To evaluate $\hat{\rho}(\bm{x},\bm{c}'_{(i,j)},\bm{r}'_{(i,j)},\underline{\bm{c}},\underline{\bm{r}})$, one can combine \eqref{eq:change_e} and \eqref{eq:change_d} in the following way:
\begin{proposition}\label{prop:cluster_purping}
	Let $\bm{x}$ be a dataset and $(\underline{\bm{c}},\underline{\bm{r}})$ a set of clusterings.
	If $\bm{c}'_{(i,j)}$ and $\bm{r}'_{(i,j)}$ are defined as in \eqref{eq:cdash} and \eqref{eq:rdash}, respectively, then it holds that
	\begin{align} \label{eq:cluster_purging}
		&\hat{\rho}\left(\bm{x},\bm{c}'_{(i,j)},\bm{r}'_{(i,j)},\underline{\bm{c}},\underline{\bm{r}}\right) \ge 1 \nonumber\\ \Leftrightarrow\\&d(x_j,r_{c_{i,j}}) \ge \frac{h(\bm{c}'_{(i,j)}) \shortminus h(\bm{c}_i)}{\shortminus \kappa_{i}} \nonumber
	\end{align}
	where $\kappa_{i}$ is the slope of the rate-distortion hull between $d(\bm{x},\bm{r}_{i\shortminus1})$ and $d(\bm{x},\bm{r}_{i})$, with $i\neq1$.
\end{proposition}
Note that $i\neq1$ in Proposition~\ref{prop:cluster_purping} is necessary since there is no slope $\kappa_0$ left of $\bm{r}_1$ in the rate-distortion hull.
\begin{proof}
	Inserting \eqref{eq:rd_curve} into the left expression of \eqref{eq:cluster_purging} gives
	\begin{equation}\label{eq:proof_step1}
		\left(\kappa_\ell\cdot d(\bm{x},\bm{r}'_{(i,j)}) + \delta_\ell\right) \:/\: h(\bm{c}'_{(i,j)}) \ge 1
	\end{equation}
	where $\ell$ is the index of the slope and vertical intercept at $d(\bm{x},\bm{r}'_{(i,j)})$.
	Since it holds that $d(\bm{x},\bm{r}'_{(i,j)}) \leq d(\bm{x},\underline{\bm{r}}_i)$ and due to the convexity of $\mathcal{L}(\cdot)$, we can assume without loss of generality that $\ell=i$.
	Then, inserting \eqref{eq:delta} into \eqref{eq:proof_step1} and factorizing $\kappa_{i}$ gives
	\begin{equation}\label{eq:proof_step2}
		\left(\kappa_{i}\cdot\left(d(\bm{x},\bm{r}'_{(i,j)}) - d(\bm{x},\bm{r}_i) \right)+h(\bm{c}_i)\right) \:/\: h(\bm{c}'_{(i,j)})\ge 1.
	\end{equation}
	Finally, after inserting \eqref{eq:change_d} into \eqref{eq:proof_step2}, the resulting expression can easily be rearranged into the right side of \eqref{eq:cluster_purging}.
\end{proof}
The main point of Prop.~\ref{prop:cluster_purping} is that $\hat{\rho}(\bm{x},\bm{c}'_{(i,j)},\bm{r}'_{(i,j)},\underline{\bm{c}},\underline{\bm{r}})$ can be easily computed from the available clusterings.
A visual intuition of how $\hat{\rho}(\cdot)$ is computed can be seen in Fig.~\ref{fig:rho}.
A concrete algorithm is described in Section~\ref{sec:algo}.
Computational speedups implied by \eqref{eq:change_e} and \eqref{eq:cluster_purging} are discussed in Section~\ref{sec:efficiency}.

\section{Practical Aspects}\label{sec:practice}
After formalizing the theoretical background needed to efficiently perform Cluster Purging, we now address several practical issues and formulate concrete algorithms for an efficient computation.

\subsection{Interpretation}
Recall that any clustering is a representation of the raw data, and that a cluster is a representation of the data assigned to it.
In essence, the theoretical foundation of Cluster Purging concerns itself with the representivity of clusterings.
If a cluster would represent its data better if one of them were removed (purged), then that deviating observation is considered an outlier.
To make the concept of representivity more tangible, we address four critical questions that may be non-obvious to the reader.

\subsubsection{How can rate-distortion outliers be interpreted?}

\begin{figure*}[!th]
	\centering
	\includegraphics[scale=0.725]{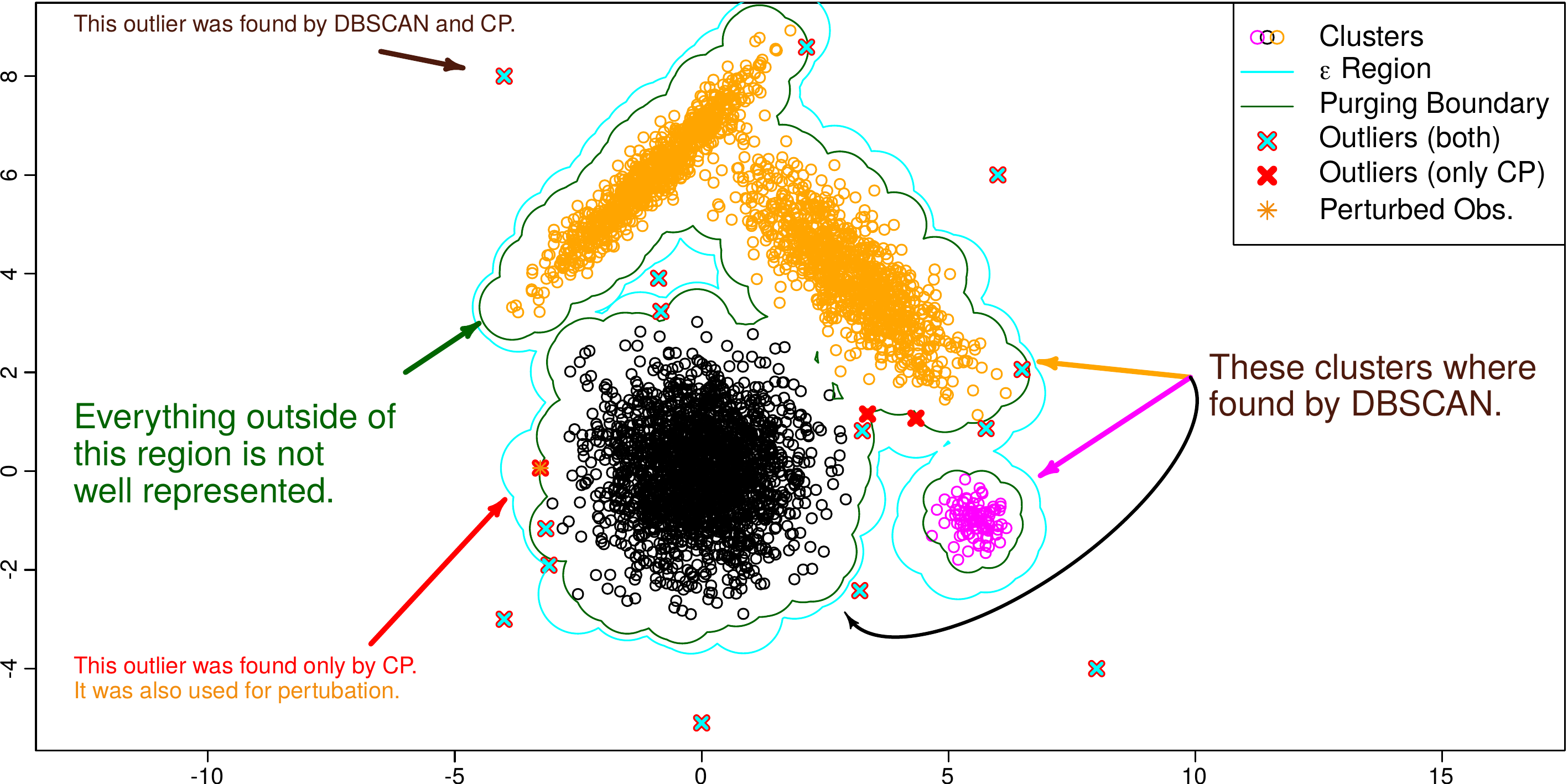}
	\caption{Cluster Purging (CP) based on DBSCAN with $\varepsilon=0.8$, \textrm{minPts}$=20$ and a $\max$-$\max$ perturbation (cf. Section~\ref{sec:eval_perturbation}	). The parametrization of DBSCAN is suboptimal, and the clustering representation can be improved by purging (i.e. uniquely encoding) outliers detected by CP. Overlapping purging boundaries were depicted as union of discs for readability. Note that observations within the $\varepsilon$ region and purging boundary may also be outliers if they are alone in their cluster (cluster size = 1).}
	\label{fig:compareDBCP}
\end{figure*}

In simple terms, a rate-distortion outlier is an observation that is ``far away'' from its cluster.
How ``far'' this needs to be is determined by a threshold that we call \textit{purging boundary}.
This purging boundary is inferred from cluster sizes and distortions across multiple clusterings, as well as from the raw dataset (see Eq.~\eqref{eq:cluster_purging}).
Hence, an accurate interpretation of rate-distortion outliers depends on how these quantities are measured.
For example, under Manhattan distances and a $k$-means clustering, all purging boundaries are hypercubes that are centered at the cluster's centroid and enclose inliers.
For DBSCAN and Euclidean distance, every observation within a specific cluster is surrounded by a hypersphere that encloses its nearest neighbor unless it is an outlier.
See Fig.~\ref{fig:compareDBCP} for a visualization.

In the context of high-dimensional data, interpretability is often addressed via dimensionality reductions such that every outlier can be described by a small subset of the original dimensions, see~\cite{dang2014discriminative,liu2017contextual}.
Similarly, rate-distortion outliers can be characterized by their low-entropy representation: They are observations that make the representation unnecessarily complicated.

\subsubsection{How is Cluster Purging different from distance-based outlier detection with clustering?}
Cluster Purging permits setups, e.g. centroid-based clustering and Euclidean distortion, that are very similar to conventional distance-based outlier detection methods such as \cite{pamula2011outlier,he2002outlier}.
The main difference between Cluster Purging and such methods is that purging boundaries are inferred based on a different clustering, and not based on a parameter.
Further, Cluster Purging is not limited to distance-based setups and is compatible with any well-defined dissimilarity measure and clustering technique, e.g. Kullback-Leibler divergence~\cite{kullback1951information} paired with fuzzy C-means clustering~\cite{dunn1973fuzzy}.

\subsubsection{Isn't Cluster Purging just another clustering-based outlier detection technique that fails if the clustering is bad?}
Not necessarily. Cluster Purging considers the original raw data via \eqref{eq:cluster_purging} in addition to all available clusterings.
Further, the rate-distortion hull~\eqref{eq:rd_curve} allows one to determine which clusterings among the available ones are best in terms of rate-distortion theory.
If all available clusterings are ``bad'', then Cluster Purging may fail to find correct outliers, yet if a single ``good'' clustering is available, then Cluster Purging will identify this clustering and use it for outlier detection.

\subsubsection{Can outliers really be detected via representivity? It seems strange that whether data are outliers depends on the size of their cluster.}
We describe a short example where rate-distortion theory-based representivity is intuitive for outlier detection:
A group of 100 people is asked to form small ``parties'' to represent their political opinions.
95 people consider themselves \textit{moderate} and form a moderate party, whereas 4 people form an \textit{extremist} party and 1 person has no opinion.
If this 1 person joined the small extremist party (clustering A), then this would have a more noticeable (outlying) effect on this party's political orientation than if the 1 person joined the large moderate party (clustering B).
Likewise, purging boundaries grow logarithmically as clusters become larger (see Eq.~\eqref{eq:change_e}). 
\subsection{Algorithms for Cluster Purging}\label{sec:algo}
\subsubsection{Parameter-free Cluster Purging}
From the theoretical formulations in Section~\ref{sec:theory}, one can directly derive an algorithm for Cluster Purging.
This algorithm takes a dataset $\bm{x}$ and a set of clusterings $(\underline{\bm{c}},\underline{\bm{r}})$ as input and returns a set of outliers without requiring any additional parameters.
In simple terms, this algorithm can be summarized as
\begin{enumerate}
	\item Compute the entropy and distortion of all clusterings.
	\item Find the lower convex hull of the resulting entropy-distortion pairs to construct a rate-distortion hull.
	\item For every cluster in every clustering on this rate-distortion hull, compute how the entropy would change if an observation in this cluster were removed.
	\item Based on the resulting changes of entropy and the slope of the rate-distortion hull, compute how much the distortion must change to pass the ``purging boundary''.
	\item Data that, when purged, would be outside of the purging boundary, as well as clusters of size 1, are outliers.
\end{enumerate}
A visual intuition of how this computation is performed is depicted in Figs. \ref{fig:rho} and \ref{fig:compareDBCP}, whereas pseudo-code for this algorithm is listed in Algorithm~\ref{alg:cluster_purging}.
An $R$ implementation can be found online\footnote{\url{https://tinyurl.com/f59ezjhk}}.
\begin{algorithm}
	\caption{Parameter-free Cluster Purging}
	\begin{algorithmic}[1]
		\REQUIRE $\bm{x}$,$\underline{\bm{c}},\underline{\bm{r}}$
		\STATE outliers $\leftarrow\emptyset$
		\FOR {clustering $(\bm{c},\bm{r})\in(\underline{\bm{c}},\underline{\bm{r}})$}
		\STATE Compute $h(\bm{c})$ according to \eqref{eq:entropy};
		\STATE Compute $d(\bm{x},\bm{r})$ according to \eqref{eq:distortion};
		\ENDFOR
		\STATE Set $\mathcal{L}$ to the lower convex hull of all $h$ and $d$;
		\STATE Compute $\kappa$ (the slopes of $\mathcal{L}$) via linear interpolation;
		\STATE Drop clusterings that are not on $\mathcal{L}$;
		\STATE Sort clusterings increasingly according to $d(\bm{x},\bm{r})$;
		\STATE Drop clustering with highest entropy (cf. Prop.~\ref{prop:cluster_purping});

		\FORALL{$(\bm{c},\bm{r})$}
		\FOR{cluster $g\in(\bm{c},\bm{r})$}
		\STATE Compute change of entropy  according to \eqref{eq:change_e};
		\ENDFOR
		\ENDFOR
		\FOR{$j\in(1,\mathellipsis,n)$}
		\IF {any side of \eqref{eq:cluster_purging} holds for all ($\bm{c},\bm{r}$)}
		\STATE outliers $\leftarrow$ outliers $\cup\; x_j$;
		\ENDIF
		
		\ENDFOR
		
		\RETURN outliers
	\end{algorithmic}
	\label{alg:cluster_purging}
\end{algorithm}
Note that the selected distortion measure $d(\cdot)$ should be equal to the distortion measure that was used to compute clusterings, e.g. for $k$-means clustering $d(\cdot)$ should be Euclidean distance, for DBSCAN it should be nearest neighbor distance.
We confirmed this insight in preliminary experiments, where it turned out that heterogeneous distortion pairs were inferior to homogeneous distortion pairs in all settings we tested.
\subsubsection{Parametric Cluster Purging}
In some settings, it may be desirable to tune cluster purging to a specific dataset.
While the parameter-free nature of the theoretical formulation of Cluster Purging prevents this, one can ``cheat'' by replacing the estimate of cluster representivity $\hat{\rho}(\cdot)$  with its true value $\rho(\cdot)$.
Of course, $\rho(\cdot)$ is not known, yet in supervised settings it can be learned from a training set, or a user may simply guess its value or use a default parametrization.

In particular, the concrete value of $\rho(\cdot)$ at a specific clustering $(\bm{c},\bm{r})$ is not even needed.
According to \eqref{eq:cluster_purging}, it is sufficient if slope $\kappa$ of the rate-distortion function at $d(\bm{x},\bm{r})$ is passed as parameter, since the remaining quantities needed to perform Cluster Purging can be easily inferred from $\kappa$.
A concrete algorithm is listed in Algorithm~\ref{alg:godsent}.
\begin{algorithm}
	\caption{Parametric Cluster Purging}
	\begin{algorithmic}[1]
		\REQUIRE $\bm{x},\bm{c},\bm{r},\kappa$
		\STATE outliers $\leftarrow\emptyset$
		\FOR{cluster $g\in(\bm{c},\bm{r})$}
		\STATE Compute change of entropy $\Delta_g$ according to \eqref{eq:change_e};
		\ENDFOR
		\FOR {$j\in(1,\mathellipsis,n)$}
		\IF {$d(x_j,r_{c_j})\cdot \kappa\le{\Delta_{c_j}}$}
		\STATE outliers $\leftarrow$ outliers $\cup\; x_j$;
		\ENDIF
		\ENDFOR
		\RETURN outliers
	\end{algorithmic}
	\label{alg:godsent}
\end{algorithm}

A clear advantage of this parametric variant of Cluster Purging is that, if the true slope is passed to the algorithm, it will necessarily be superior to the parameter-free variant.
Further, this variant only needs a single clustering, and is very simple overall.
However, we believe that the parameter-free algorithm should generally be preferred over its parametric counterpart (cf. \cite{keogh2004towards}).

\subsection{Efficiency}\label{sec:efficiency}
In the pseudo-code of Algorithms~\ref{alg:cluster_purging} and \ref{alg:godsent} there are several verbose instructions whose computational complexity might be non-obvious.
In Algorithm~\ref{alg:cluster_purging}, lines 3 and 4 require $\mathcal{O}(n)$ steps, whereas all remaining verbose steps in both algorithms require at most $O(\nu td)$ steps.
Asymptotically, $\nu$ is the largest number of clusters, $t$ the number of clusterings, and $d$ the dimensionality of the dataset.
Since all three of these quantities were assumed to be constant, these steps can hence be performed in $\mathcal{O}(1)$ time.
Consequently, the time complexity of both Algorithms can be reduced to $\mathcal{O}(n)$.

In terms of space complexity, one will naturally require at least $\mathcal{O}(tn)$ space to store all clusterings.
The remaining memory overhead of both algorithms is constant.

\subsection{Obtaining Multiple Clusterings $(\underline{\bm{c}},\underline{\bm{r}})$}\label{sec:practice_perturbation}
In recent years, datasets have become increasingly large and ``\textit{in many situations, the knowledge extraction process has to be very efficient and close to real time because storing all observed data is nearly infeasible}''~\cite{wu2013data}.
Consequently, it may occur in practice that computing multiple good clusterings of a dataset may be too costly, although the above formulation of rate-distortion hulls would require this.
To address this issue, we here discuss methods for efficiently obtaining similar clusterings, i.e. perturbations, from a single ``seed'' clustering.

In general the theoretical formulations of Cluster Purging permit arbitrary perturbations.
However, the quality of a clustering representivity estimate depends on how ``strongly'' the seed clustering was perturbed.
Hence, from a rate-distortion theoretic perspective, it is desirable that clustering $(\bm{c},\bm{r})$  and its perturbation $(\bm{\tilde{c}},\bm{\tilde{r}})$ are as similar as possible, yet not identical.
To achieve this, it is typically sufficient to modify the cluster assignment and representation of a single observation $x_j$, given that this change results in a different entropy-distortion pair, i.e. $[h(\bm{c}),d(\bm{x},\bm{r})]\neq[h(\bm{\tilde{c}}),d(\bm{x},\bm{\tilde{r}})]$.
A concrete change that causes this is typically given by selecting the cluster with the largest size, i.e. $\text{argmax}\bm{f}^{\bm{c}}$, and removing the observation that causes the largest distortion in this cluster.
At first glance, this may seem counterintuitive, since the aim of a perturbation is to cause a small yet sufficiently large change in the clustering, and hence removing the observation from the smallest cluster with the smallest distortion would seem better.
We elaborate on this and empirically compare other perturbation strategies in Section~\ref{sec:eval_perturbation}.

\subsection{Nearest Neighbor Representations}

A further issue may occur when the selected clustering technique, e.g. DBSCAN, yields cluster assignments $\bm{c}$ yet no representations $\bm{r}$.
In such cases, one can jointly infer $\bm{r}$ from $\bm{x}$ and $\bm{c}$ based on the following intuition:
Since clustering techniques group data according to some similarity measure~\cite{jain1999data}, this similarity measure implicitly contains information on what a representation for such a clustering technique might be.
In the case of DBSCAN, which clusters data according to nearest neighbor distances, one can simply represent every $x_j$ by its nearest neighbor within the cluster of $x_j$.
While using such representations leads to no compression of the data, this is still meaningful if one wants to detect outliers.
We demonstrate this empirically in Section~\ref{sec:empirical_eval}, whereas a visualization can be seen in Fig.~\ref{fig:db}.

\begin{figure}[!t]
	\includegraphics[scale=0.75]{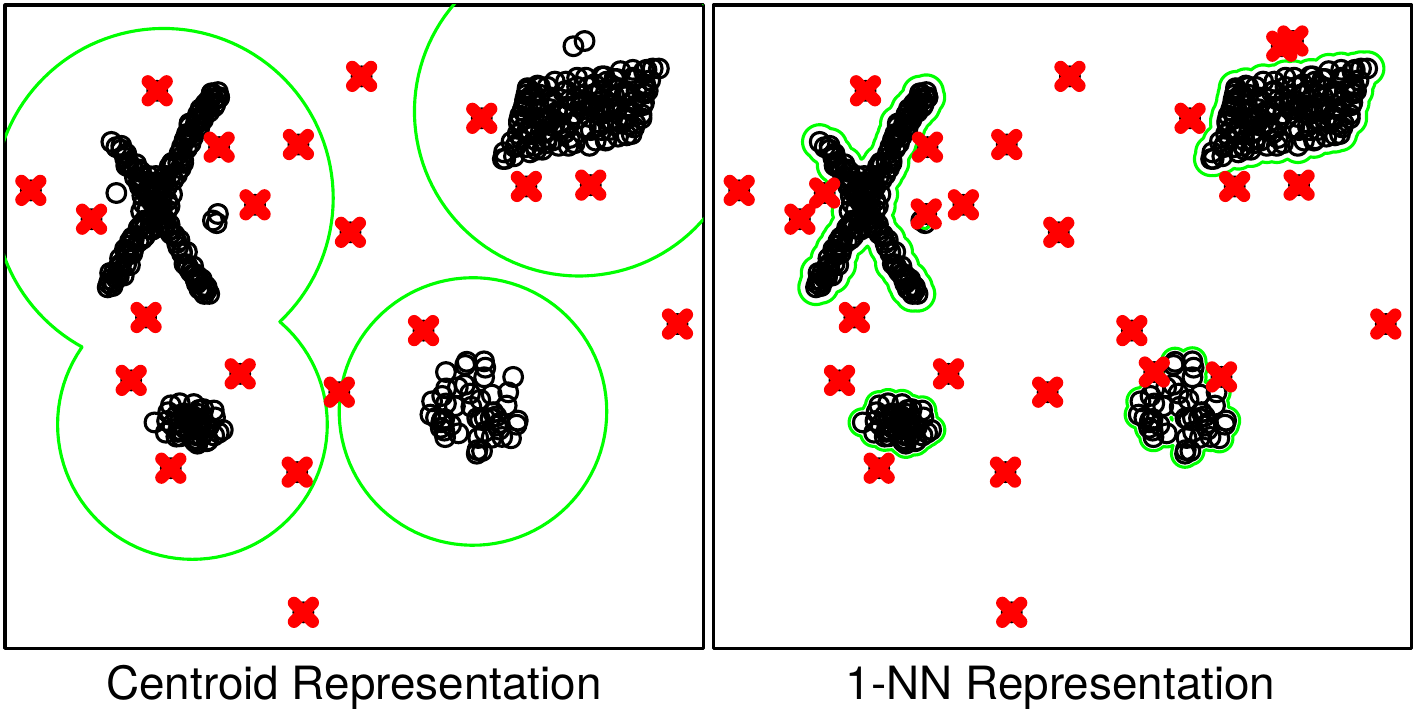}
	\caption{Cluster Purging applied to a synthetic dataset~\cite{bohm2009coco} clustered with DBSCAN. Detected outliers are depicted in red ($\times$). \textit{Left}: For every cluster, a single Euclidean centroid was used as representative, resulting in large, spherical purging boundaries. \textit{Right}: For every observation, its nearest neighbor within the same cluster was used as representative, resulting in tight boundaries that fit the data well. }
	\label{fig:db}
\end{figure}

\subsection{Rules of Thumb}
Since Cluster Purging allows highly diverse setups, we formulate three rules of thumb for guiding practitioners:

First, different clusterings offer different entropy-distortion trade-offs, e.g. a clustering with $n$ clusters leads to a lossless representation yet no compression, whereas a representation with a single cluster leads to good compression yet large distortion.
Since purging boundaries depend on cluster sizes, they will adapt to different entropy-distortion trade-offs.
Generally speaking, Cluster Purging will work well under many different trade-offs as long as one avoids the extremes of the empirical rate-distortion function.

Secondly, it is desirable that the selected clusterings and/or perturbations have similar entropy-distortion trade-offs.
The reason for this is that the estimated rate-distortion slope between two clusterings becomes less accurate the further these clusterings are apart in rate-distortion space.
Hence, it is generally not a good idea to combine different clustering techniques, e.g. $k$-means and DBSCAN.
Pairing similar clusterings is usually better, e.g. $7$-means with $8$-means.
Fixing a single clustering $(\bm{c},\bm{r})$ and computing a slight perturbation $(\tilde{\bm{c}},\tilde{\bm{r}})$ by changing the cluster assignment of a single observation is likely best.

Thirdly, the selected distortion measure should be related to the selected clustering technique.
For instance, it is often better to pair $k$-means with Euclidean distortion than with Hamming distortion, and for hierarchical clusterings one should use the same distance function for computing the clustering and for measuring distortion.
For probabilistic clustering techniques, distortion should likely be measured via Kullback-Leibler divergence.

\subsection{Limitations}
The concept of rate-distortion outliers describes \textit{individual} observations that are outlying.
Collective outliers~\cite{chandola2009anomaly} and outlying clusters are not covered and will be addressed in future work. 
Further, in rare cases it may occur that the computed rate-distortion hull has an increasing segment.
In such an increasing region \eqref{eq:cluster_purging} does not hold, and it is best to ignore this region of the rate-distortion hull.
Finally, while Algorithms~\ref{alg:cluster_purging} and \ref{alg:godsent} can be computed in $\mathcal{O}(n)$ time, the computation of the clusterings they are based on may be more costly.

\section{Experimental Evaluation}\label{sec:experiments}
To evaluate the practical applicability and correctness of rate-distortion theory for outlier detection, we conduct a case study in which different perturbation strategies are analyzed (Section~\ref{sec:eval_perturbation}).
In Section~\ref{sec:empirical_eval}, we compare our method Cluster Purging (CP) with other state-of-the-art outlier detection methods in an experimental evaluation on benchmark datasets.
Further, we also analyze how frequently Cluster Purging improves upon outliers detected by an existing clustering.
Throughout all experiments, we use Euclidean distance as distance measure in all clustering techniques, and consequently also as distortion measure.
We avoid using non-distance distortion measures such as Kullback-Leibler divergence, since this would make a fair comparison of Cluster Purging with distance-based outlier detectors difficult.
Centroids are computed as the arithmetic mean of all observations in a cluster whenever needed.
The \textbf{source code} for reproducing all results, as well as all data can be accessed online\footnote{\url{https://tinyurl.com/f59ezjhk}}.

\subsection{Case Study: Perturbation for Map Denoising}\label{sec:eval_perturbation}

\begin{figure}[!t]
	\includegraphics[scale=0.75]{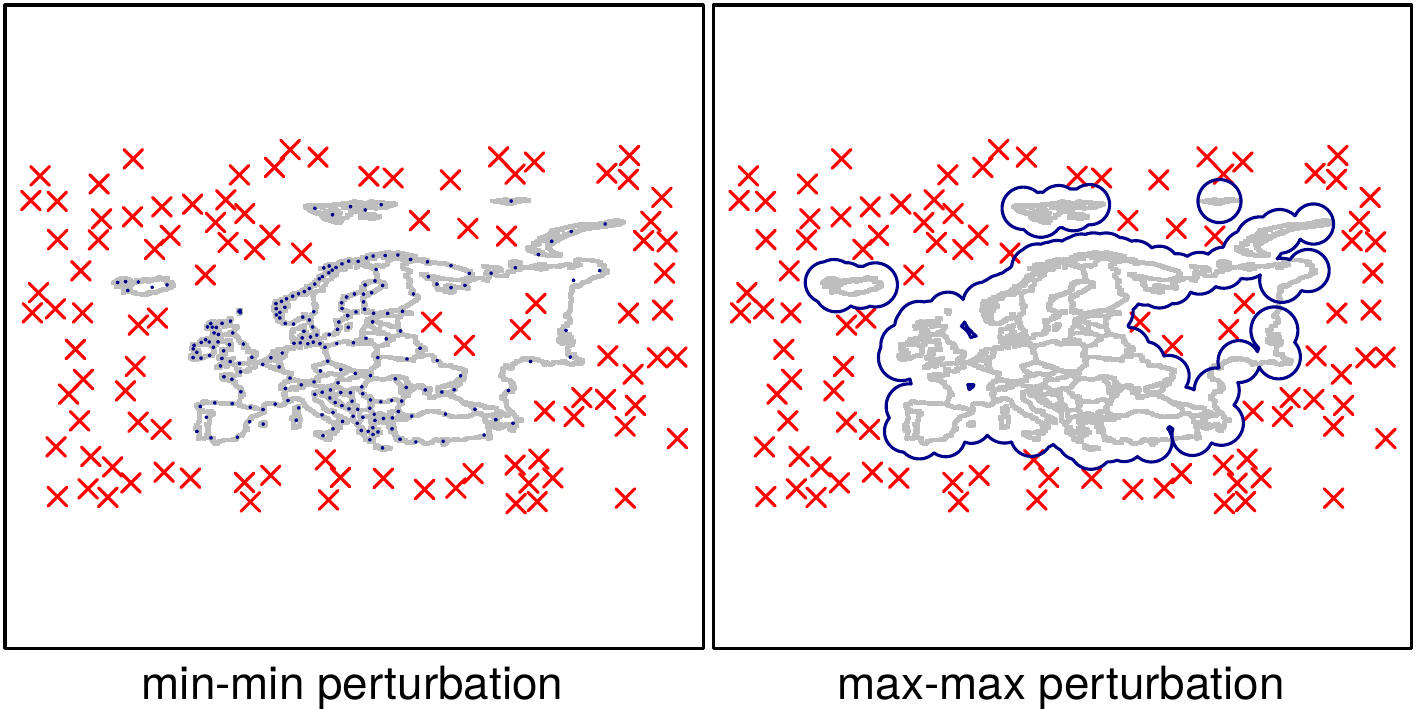}
	\caption{Case Study: Comparison of Purging Boundaries (blue) with $\min$-$\min$ perturbation and $\max$-$\max$ perturbation. True noise points are depicted in red ($\times$), while detected outliers are not depicted for readability (the left plot would be covered in outliers). \textit{Left}: Purging boundaries  derived from a $\min$-$\min$ perturbation are so small that they are barely visible. \textit{Right}: Purging boundaries derived from a $\max$-$\max$ perturbation are $\approx 67$ times larger than $\min$-$\min$ purging boundaries, almost fully covering the map of Europe.}
	\label{fig:europe}
\end{figure}

From the elaborations made in Section~\ref{sec:practice_perturbation}, one can derive four different perturbation strategies\footnote{In all four perturbation strategy descriptions, ``purge'' is short for ``reassign to additional unique cluster''.}
\begin{enumerate}
	\item $\min$-$\min$: Select smallest cluster, purge least distorted observation.
	\item $\min$-$\max$: Select smallest cluster, purge most distorted observation.
	\item $\max$-$\min$: Select largest cluster, purge least distorted observation.
	\item $\max$-$\max$: Select largest cluster, purge most distorted observation.
\end{enumerate}

We compare all four strategies in a case study, where the goal is to denoise a dataset via $k$-means clustering and outlier detection. The dataset contains coordinates of a map of the continent Europe~\cite{franti2018k} with $100$ artificially added noise points.
Since $k$-means clustering algorithms are sensitive to the selected initial centers, we fix the number of centroids to $k=225$, and compute $1000$ different initializations,  each for $10$ different initial random seeds.
For every computed clustering, we perform Cluster Purging based on all $4$ perturbation strategies with noise points considered as outliers.
As evaluation measure, we use $F_1 = 2\cdot\frac{\text{precision}\cdot\text{recall}}{\text{precision}+\text{recall}}$.
Further, since inlier and outlier classes are heavily imbalanced ($169673:100$) we compute average class-wise $F_1$-scores in addition to average raw $F_1$-scores.
The results of this case study are reported in Table~\ref{tab:case_study}, whereas a visualization can be seen in Fig.~\ref{fig:europe}.

\begin{table}[!t]
	\centering
	\caption{Case Study: Average Class-Wise $F_1$-scores}
	\label{tab:case_study}
	\begin{tabular}{r|cccc}
		&\multicolumn{4}{c}{Perturbation Strategy}\\\hline
	Measure	&$\min$-$\min$&$\min$-$\max$&$\max$-$\min$&$\max$-$\max$ \\\hline
	Outlier $F_1$-score&\textbf{0.17} &  0.08 &  0.00 &  0.16 \\
	Inlier $F_1$-score&0.43 &  0.94 &  0.16 &  \textbf{0.97} \\\hline
	Combined $F_1$-score&0.30 &  0.51 &  0.08  & \textbf{0.56}\\
	\end{tabular}
	
\end{table}

\subsection{Competitive Evaluation on Benchmark Datasets}\label{sec:empirical_eval}

\subsubsection{Setup}
We compare both variants of our method, Cluster Purging (CP) and Parametric Cluster Purging (CPP) against closely related outlier detection methods mentioned in Section~\ref{sec:related_work}:
\begin{itemize}
	\item The one-class rate-distortion model (OCRD)~\cite{crammer2008rate}.
	\item Variants of $k$-means that detect outliers, i.e. $k$-means$\shortminus\shortminus$ (KM$\shortminus\shortminus$)~\cite{chawla2013k} and $k$-means with outlier removal (KMOR)~\cite{gan2017k}.
	\item Raw clusterings, i.e. $k$-means clustering, Hierarchical Agglomerative Clustering (HAC) with complete linkage and DBSCAN~\cite{ester1996density}, with singelton clusters considered as outliers (these variants are referred to as \textit{Vanilla} detectors)
	\item Cluster-based local outlier factor (CBLOF)~\cite{jiang2008clustering} based on all vanilla clusterings and raw local outlier factor (LOF)~\cite{breunig2000lof}.
	\item Outlier detection for high-dimensional data via Local Projection Score (LPS)~\cite{liu2017efficient}.
	\item Cluster Purging (CP) with a single $\max$-$\max$ perturbation and Parametric Cluster Purging (CPP), both based on all vanilla clustering techniques ($t=1$ clustering each). Other perturbation methods are addressed in Section~\ref{sec:perturbation_results}.
\end{itemize}
We omit \cite{ruiz2007c,smiti2013soft} since they use soft clusterings; \cite{pamula2011outlier} and\cite{he2002outlier} because they have high computational cost and are not reproducible, respectively;\cite{papadimitriou2003loci,duan2007local,kriegel2009loop} since we found that two variants of the Local Outlier Factor are sufficient.
To enable a comparison with LOF, CBLOF and LPS, which return outlier scores instead of outliers indices, we take the top $m=|\bm{y}|$ scores of these methods, where $m$ is the true number of outliers in dataset $\bm{x}$.
As evaluation measure, we use $F_1$-score.
Further, since all clustering algorithms under consideration (and most outlier detectors) have parameters, it is difficult to generalize outlier detection performances based on a single arbitrarily selected parametrization.
Hence, the parameters of all clustering techniques (and outlier detection methods) are grid searched over their respective parameter space towards maximizing \textit{$F_1$-score}.
For methods having several parameters where a grid search would be infeasible, some parameters are set according to literature recommendations.
The detailed grid search setups and parametrizations are listed in Table~\ref{tab:methods}.

\begin{table}[!t]
	\caption{Compared outlier detection methods and their parametrizations}
	\label{tab:methods}
	\begin{tabular}{l|l|l}
		\textbf{Methods} & \textbf{Grid searched parameters} &\textbf{Hard coded parameters}\\
		\hline
		OCRD&$\beta$: $[0.1,\mathellipsis,10]$ ($n$ steps)&$q(0)=0.5$\\
		&&uniform prior\\
		\hline
		\multicolumn{3}{c}{$k$-means}\\
		\hline
		Vanilla&$k=[2,\mathellipsis,10]$&$n_\text{start}$=1000\\
		KM$\shortminus\shortminus$&$k=[2,\mathellipsis,10]$&$n_{\text{outlier}}=m$\\
		KMOR&$k=[2,\mathellipsis,10]$, &$\delta=1$\\
		&$\gamma=[0.1,\mathellipsis,10]$ ($n$ steps)&\\
		CBLOF&{\scriptsize (vanilla parameters)\par}&$b=\min(k\shortminus1,5)$\\
		CP&{\scriptsize (vanilla parameters)\par}&\\
		CPP&{\scriptsize (vanilla parameters)\par},&\\
		&$\kappa =[0.1,\mathellipsis,10]$ ($n$ steps)&\\
		\hline
		\multicolumn{3}{c}{HAC}\\
		\hline
		Vanilla &$k=[1,\mathellipsis,n]$&\\
		CBLOF&{\scriptsize (vanilla parameters)\par}&$b=\min(k\shortminus1,5)$\\
		CP&{\scriptsize (vanilla parameters)\par}&\\
		CPP&{\scriptsize (vanilla parameters)\par},&\\
		&$\kappa= [0.1,\mathellipsis,10]$ ($n$ steps)&\\
		\hline
		\multicolumn{3}{c}{DBSCAN}\\
		\hline
		Vanilla&$\min_p$=$[d+1,\mathellipsis,d+10]$&\\
		&$\varepsilon=$ unique $\min_p$-NN dists.&\\
		CBLOF&{\scriptsize (vanilla parameters)\par}&$b=\min(k\shortminus1,5)$\\
		CP&{\scriptsize (vanilla parameters)\par}&\\
		CPP&{\scriptsize (vanilla parameters),\par}&\\
		&$\kappa= [0.1,\mathellipsis,10]$ ($n$ steps)&\\
		\hline
		\multicolumn{3}{c}{No Clustering}\\
		\hline
		LOF & $k=[1,\mathellipsis,n-1]$&\\
		LPS & $k=[2,\mathellipsis,\lceil\frac{d}{2}\rceil]$ & $n_{\text{outlier}}=m$
	\end{tabular}
\end{table}

Additionally, to evaluate the claimed computational efficiency of CP and CPP, we track the average runtime of each method per call.
We report this quantity instead of overall runtime since the total number of needed calls to each outlier detection method varies for each grid search.

\subsubsection{Datasets}
The experimental evaluation of all detectors is performed on 13 publicly available benchmark datasets, taken from~\cite{campos2016evaluation}.
These datasets come from diverse domains such as medicine, space, and telecommunications, and were commonly used as benchmarks in literature.
More detailed descriptions of the domain background of these datasets can be found in~\cite{campos2016evaluation}.
For this experimental evaluation, dataset Arrhymthia is particularly noteworthy since it is high-dimensional with $n\approx d$, and Heart, Pima and Ionosphere since they have an outlier ratio $\frac{m}{n}$ close to $50\%$.
\subsubsection{Main Results}

The main results of the competitive evaluation  are depicted in Table~\ref{tab:results}.
Overall, detectors based on $k$-means clusterings performed worse than detectors based on other clusterings.
The overall highest average $F_1$-score was achieved by CBLOF based on HAC clustering.
For other clustering methods, CPP performed best.
The average performance of OCRD, which is bound to a Blahut-Arimoto-like clustering, was competitive with detectors based on $k$-means clusterings, yet lower than that of detectors based on HAC and DBSCAN.

Regarding computational efficiency, vanilla clusterings were faster than methods based on these clusterings.
The fastest method was vanilla $k$-means, while CPP had the overall lowest surplus runtime after its clustering was computed.
The slowest method was LOF followed by LPS.

When considering on how many datasets detectors with exchangeable clusterings did not perform worse than the respective vanilla clustering, there is a clear ranking.
Our method CPP performed best (100\%), followed by CP (85\%), followed by CBLOF (62\%).

\begin{table*}[!th]
	\caption{Competitive Evaluation Results.}
	\scriptsize
	\centering
	\setlength{\tabcolsep}{3pt}
	\begin{tabular}{r|c|cccccc|cccc|cccc|cc}
		\hline
		Clustering&\multicolumn{1}{c}{B-A}&\multicolumn{6}{c}{k-means}&\multicolumn{4}{c}{HAC}&\multicolumn{4}{c}{DBSCAN}&\multicolumn{2}{c}{None}\\
		\hline
		Detector& OCRD & Vanilla&KM$\shortminus \shortminus$ & KMOR&CBLOF & CP & CPP & Vanilla & CBLOF & CP & CPP& Vanilla & CBLOF & CP & CPP&LOF&LPS\\ 
		\hline
		Adapts \#outlier? &\checkmark&\checkmark&$\times$&\checkmark&$\times$&\checkmark&\checkmark&\checkmark&$\times$&\checkmark&\checkmark&\checkmark&$\times$&\checkmark&\checkmark&$\times$&$\times$\\\hline
		Parameter-free?&$\times$&$\times$&$\times$&$\times$&$\times$&\checkmark&$\times$&$\times$&$\times$&\checkmark&$\times$&$\times$&$\times$&\checkmark&$\times$&$\times$&$\times$\\\hline
		\multicolumn{18}{c}{F\textsubscript{1}-score}\\\hline
		Arrhymthia & 0.68 & 0.01 & 0.67 & 0.63 & 0.63 & 0.20 & 0.69 & 0.68 & 0.67 & 0.70 & 0.71 & 0.62 & 0.66 & 0.62 & 0.69 & 0.69 & 0.60 \\ 
		Heart & 0.65 & 0.00 & 0.57 & 0.62 & 0.53 & 0.16 & 0.63 & 0.64 & 0.56 & 0.64 & 0.65 & 0.63 & 0.54 & 0.63 & 0.67 & 0.55 & 0.48 \\ 
		Hepatitis & 0.43 & 0.00 & 0.23 & 0.41 & 0.31 & 0.24 & 0.34 & 0.31 & 0.31 & 0.32 & 0.36 & 0.35 & 0.31 & 0.35 & 0.35 & 0.31 & 0.23 \\ 
		Parkinson & 0.86 & 0.00 & 0.80 & 0.86 & 0.78 & 0.12 & 0.79 & 0.86 & 0.86 & 0.86 & 0.86 & 0.81 & 0.82 & 0.81 & 0.86 & 0.78 & 0.73 \\ 
		Pima & 0.60 & 0.00 & 0.49 & 0.56 & 0.47 & 0.20 & 0.55 & 0.52 & 0.50 & 0.52 & 0.56 & 0.54 & 0.47 & 0.53 & 0.54 & 0.54 & 0.43 \\ 
		Stamps & 0.59 & 0.00 & 0.29 & 0.51 & 0.45 & 0.16 & 0.38 & 0.24 & 0.94 & 0.33 & 0.52 & 0.64 & 0.42 & 0.65 & 0.65 & 0.39 & 0.65 \\ 
		Glass & 0.18 & 0.00 & 0.11 & 0.24 & 0.44 & 0.24 & 0.34 & 0.32 & 0.22 & 0.33 & 0.36 & 0.33 & 0.44 & 0.33 & 0.33 & 0.33 & 0.11 \\ 
		Ionosphere & 0.69 & 0.00 & 0.82 & 0.77 & 0.67 & 0.51 & 0.80 & 0.86 & 0.75 & 0.84 & 0.87 & 0.77 & 0.85 & 0.77 & 0.88 & 0.83 & 0.67 \\ 
		Lympho & 0.86 & 0.00 & 0.33 & 0.40 & 0.17 & 0.67 & 0.80 & 0.67 & 0.33 & 0.83 & 0.83 & 0.29 & 0.67 & 0.55 & 0.62 & 0.83 & 0.33 \\ 
		Shuttle & 0.32 & 0.00 & 0.23 & 0.21 & 0.15 & 0.11 & 0.20 & 0.21 & 0.85 & 0.21 & 0.27 & 0.32 & 0.15 & 0.32 & 0.34 & 0.31 & 0.31 \\ 
		WBC & 0.70 & 0.00 & 0.70 & 0.78 & 0.60 & 0.74 & 0.78 & 0.53 & 1.00 & 0.64 & 0.78 & 0.82 & 0.50 & 0.82 & 0.82 & 0.80 & 0.60 \\ 
		WDBC & 0.67 & 0.00 & 0.80 & 0.84 & 0.80 & 0.80 & 0.84 & 0.84 & 0.90 & 0.78 & 0.90 & 0.84 & 0.90 & 0.90 & 0.90 & 0.80 & 0.70 \\ 
		WPBC & 0.40 & 0.00 & 0.23 & 0.40 & 0.34 & 0.19 & 0.41 & 0.39 & 0.43 & 0.41 & 0.42 & 0.44 & 0.38 & 0.44 & 0.44 & 0.36 & 0.28 \\  
		\hline
		Average& 0.59  &  0.00  &  0.48  &  0.56  &  0.49  &  0.33  &  \textbf{0.58}  &  0.54  &  \textbf{0.64}  &  0.57  &  0.62  &  0.57  &  0.55  &  0.59  &  \textbf{0.62} & 0.58 & 0.47 \\
		
		\multicolumn{18}{c}{\textcolor{white}{Invisible}} \\ \hline
		\multicolumn{18}{c}{Average runtime per method call (milliseconds)}\\\hline
		Arrhymthia & 25.48 & 6.04 & 274.02 & 109.27 & 6.25 & 6.25 & 6.27 & 2.25 & 28.64 & 22.11 & 11.91 & 97.20 & 173.57 & 263.24 & 109.76 & 543.70 & 4065.01 \\ 
		Heart & 9.39 & 0.21 & 78.56 & 22.09 & 0.21 & 0.21 & 0.24 & 0.86 & 10.88 & 5.85 & 2.77 & 1.06 & 17.57 & 23.08 & 1.65 & 278.49 & 29.91 \\ 
		Hepatitis & 2.86 & 0.09 & 18.03 & 5.21 & 0.10 & 0.09 & 0.11 & 0.35 & 3.48 & 3.15 & 1.12 & 0.40 & 5.70 & 7.54 & 0.65 & 136.42 & 11.62 \\ 
		Parkinson & 9.97 & 0.15 & 29.96 & 16.91 & 0.16 & 0.16 & 0.19 & 0.62 & 8.53 & 5.43 & 2.28 & 0.85 & 15.26 & 17.67 & 1.48 & 222.19 & 26.45 \\ 
		Pima & 41.37 & 0.47 & 405.15 & 97.74 & 0.49 & 0.49 & 0.56 & 5.12 & 33.45 & 16.09 & 7.50 & 2.24 & 62.45 & 62.80 & 4.34 & 969.16 & 76.47 \\ 
		Stamps & 22.08 & 0.31 & 146.76 & 31.91 & 0.32 & 0.32 & 0.35 & 1.31 & 13.58 & 7.21 & 3.28 & 1.86 & 26.04 & 29.72 & 2.04 & 349.49 & 34.01 \\ 
		Glass & 13.47 & 0.15 & 65.98 & 18.01 & 0.15 & 0.15 & 0.18 & 0.68 & 7.96 & 5.29 & 2.05 & 0.80 & 16.07 & 18.33 & 1.32 & 238.27 & 24.91 \\ 
		Ionosphere & 21.40 & 0.62 & 101.73 & 43.48 & 0.63 & 0.63 & 0.67 & 1.45 & 14.15 & 8.87 & 4.10 & 6.81 & 44.66 & 44.89 & 3.00 & 358.27 & 123.49 \\ 
		Lympho & 5.85 & 0.10 & 77.28 & 11.33 & 0.10 & 0.10 & 0.13 & 0.38 & 6.62 & 3.85 & 1.75 & 0.41 & 7.55 & 12.89 & 1.07 & 172.90 & 17.45 \\ 
		Shuttle & 54.22 & 0.98 & 56.49 & 138.08 & 1.02 & 1.02 & 1.11 & 9.74 & 46.86 & 21.30 & 10.45 & 10.19 & 111.05 & 112.32 & 6.25 & 1447.65 & 95.85 \\ 
		WBC & 11.92 & 0.18 & 81.76 & 20.14 & 0.19 & 0.19 & 0.21 & 0.70 & 8.37 & 5.60 & 2.23 & 1.38 & 16.19 & 22.96 & 1.42 & 250.97 & 25.09 \\ 
		WDBC & 21.38 & 0.52 & 140.29 & 40.72 & 0.53 & 0.53 & 0.58 & 1.53 & 16.75 & 9.67 & 4.42 & 2.35 & 36.57 & 36.50 & 2.89 & 375.74 & 120.32 \\ 
		WPBC & 8.82 & 0.38 & 71.80 & 20.27 & 0.38 & 0.38 & 0.41 & 0.71 & 9.19 & 5.58 & 2.57 & 3.00 & 16.15 & 23.31 & 1.75 & 223.21 & 33.84 \\ 
		\hline
		Total average&19.09  &  \textbf{0.79}  &  119.06     &  44.24   &  0.81   & 0.81   &  0.85  &  \textbf{1.98}  &  16.04  &  9.23  &  4.34  &  \textbf{9.89}  &   42.22  &  51.94  &  10.59 &428.19 & 360.34 \\
	\end{tabular}

\vspace*{0.35cm}

Perturbation Specific Results

\begin{tabular}{l|cccc|cccc|cccc}
	
	&\multicolumn{4}{c}{$k$-means}&\multicolumn{4}{c}{HAC}&\multicolumn{4}{c}{DBSCAN}\\
	\hline
	& {\tiny $\min$-$\min$\par} & {\tiny $\min$-$\max$\par} & {\tiny $\max$-$\min$\par} & {\tiny $\max$-$\max$\par} & {\tiny $\min$-$\min$\par} & {\tiny $\min$-$\max$\par} & {\tiny $\max$-$\min$\par} & {\tiny $\max$-$\max$\par} & {\tiny $\min$-$\min$\par} & {\tiny $\min$-$\max$\par} & {\tiny $\max$-$\min$\par} & {\tiny $\max$-$\max$\par} \\ 
	\hline
	Average $F_1$ & \textbf{0.45} & 0.11 & 0.33 & 0.33 & 0.55 & 0.54 & 0.35 & \textbf{0.57} & 0.33 & \textbf{0.59} & 0.29 & \textbf{0.59} \\ 
	Average Runtime & 0.81 & 0.81 & 0.81 & 0.81 & 9.22 & 9.24 & 9.23 & 9.23 & 51.89 & 51.90 & 51.64 & 51.94 \\
\end{tabular}
	\label{tab:results}
\end{table*}

\subsubsection{Detailed Results per Perturbation Method}\label{sec:perturbation_results}
In the bottom of Table~\ref{tab:results}, average $F_1$-scores and runtimes of all four considered perturbation strategies are listed per clustering.
In terms of average $F_1$-scores, the $\max$-$\max$ perturbation scored highest most often, whereas differences in runtime between perturbation strategies are negligible.
For this reason and due to lack of space, only the detailed scores per dataset of CP with $\max$-$\max$ perturbations are listed in Table~\ref{tab:results}.

\section{Discussion}\label{sec:discussion}
The results of the case study indicate that the $\max$-$\max$ perturbation is slightly superior over the other considered perturbation strategies.
This is in accordance with the results of the competitive evaluation, and hence we overall argue that $\max$-$\max$ perturbations should be preferred.

In the benchmark evaluation, the parameter-free variant of Cluster Purging seems to be competitive with other detectors, yet does not demonstrate superior detection performances.
However, this lack of superiority may be tolerable when one considers that a parameter-free algorithm was compared against parametric ones---where CBLOF, the strongest competitor, received information on how many outliers are present in the dataset.
Of course, one may argue that Cluster Purging is not truly parameter-free if only a single clustering is provided, since the selected perturbation strategy can also be seen as a parameter.
Yet, when one considers that multiple different perturbation strategies may lead to similar detection results (cf. Table~\ref{tab:methods} $\min$-$\max$ and $\max$-$\max$), then it can be argued that Cluster Purging is still ``less'' parameter-dependent than other competing methods.
Further, if a single parameter is allowed (rate-distortion hull slope $\kappa$), then one can use the parametric variant of Cluster Purging, which overall seems to compete strongly against the state-of-the-art.
The slow runtime of the seemingly efficient method LOF can be explained by the need of computing up to $n-1$ nearest neighbors during parameter optimization.

It is also noteworthy that Cluster Purging---especially its parametric variant---performed (or was tied for) best on high-dimensional and outlier heavy datasets \textrm{Arrhymthia}, \textrm{Heart}, \textrm{Pima} and \textrm{Ionosphere}.
Hence, one can expect Cluster Purging to tolerate high-dimensional data or high outlier ratios even if clustering such data is challenging.

Consequently, we expect Cluster Purging to perform well in a variety of domains under the premise that a reasonably-working clustering technique is known.
Further, our proposed algorithms, especially the parametric variant, are efficient in terms of computational complexity, requiring only $\mathcal{O}(n)$ time.
While at least one clustering is still required as input, this efficiency can be a key advance in scenarios where prior clusterings of the data are available.

\appendices
\ifCLASSOPTIONcompsoc
  \section*{Acknowledgments}
\else
  \section*{Acknowledgment}
\fi

We thank the anonymous reviewers for their valuable feedback.
This work was partly funded by the iDev40 project.
The iDev40 project  has  received  funding  from  the  ECSEL  Joint  Undertaking (JU) under  grant  agreement  No  783163. The JU receives support from the European Union's Horizon  2020  research  and  innovation  programme. It  is  co-funded  by  the  consortium  members, grants from Austria, Germany, Belgium, Italy, Spain and Romania.

\ifCLASSOPTIONcaptionsoff
  \newpage
\fi

\bibliographystyle{IEEEtran}
\bibliography{final_manuscript}

\begin{IEEEbiography}[{\includegraphics[width=1in,height=1.25in,clip,keepaspectratio]{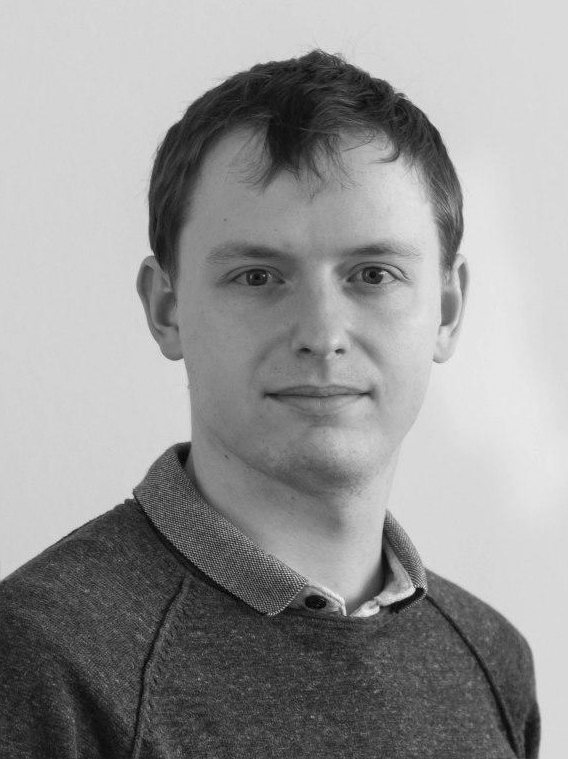}}]{Maximilian B. Toller}
is a PhD candidate at Graz University of Technology, Austria and is currently a researcher at Know-Center GmbH, Graz, Austria.
His research interests include outlier detection, time series data mining, theoretical foundations of data mining, and computational complexity theory.
\end{IEEEbiography}

\begin{IEEEbiography}
[{\includegraphics[width=1in,height=1.25in,clip,keepaspectratio]{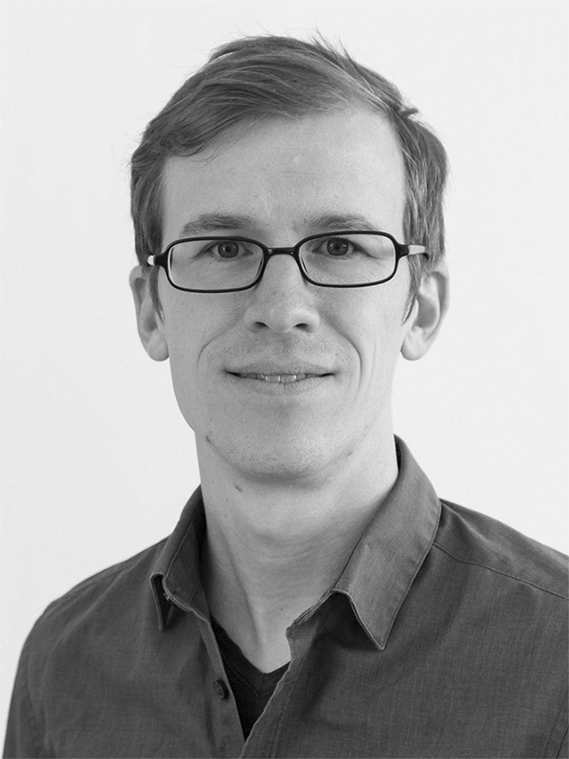}}]{Bernhard C. Geiger}
(S'07, M'14, SM'19) received the Dipl.-Ing. degree in electrical engineering (with distinction) and the Dr. techn. degree in electrical and information engineering (with distinction) from Graz University of Technology, Austria, in 2009 and 2014, respectively.

In 2009 he joined the Signal Processing and Speech Communication Laboratory, Graz University of Technology, as a Project Assistant and took a position as a Research and Teaching Associate at the same lab in 2010. He was a Senior Scientist and Erwin Schr\"odinger Fellow at the Institute for Communications Engineering, Technical University of Munich, Germany from 2014 to 2017 and a postdoctoral researcher at the Signal Processing and Speech Communication Laboratory, Graz University of Technology, Austria from 2017 to 2018. He is currently a Senior Researcher at Know-Center GmbH, Graz, Austria. His research interests cover information theory for machine learning, theory-assisted machine learning, and information-theoretic model reduction for Markov chains and hidden Markov models.
\end{IEEEbiography}

\begin{IEEEbiography}[{\includegraphics[width=1in,height=1.25in,clip,keepaspectratio]{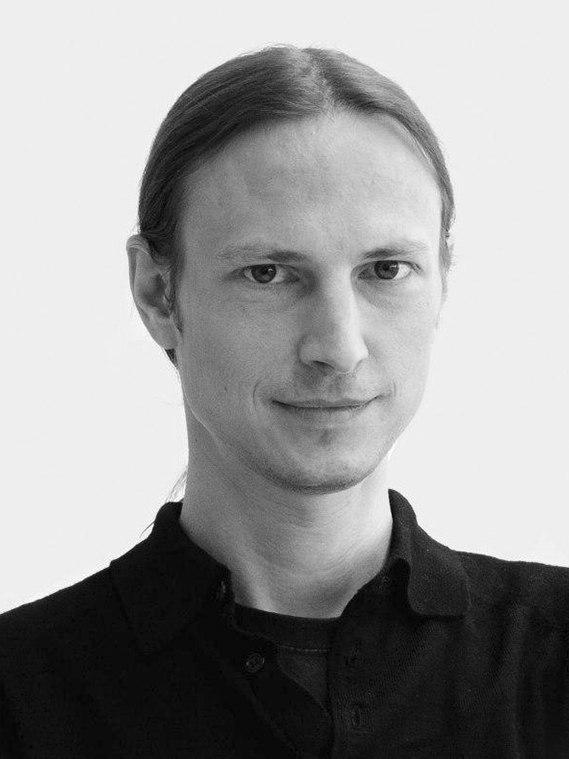}}]
	{Roman Kern}
	 is an Ass.Prof. at the Institute for Interactive Systems and Data Science at the Technical University of Graz and head of Knowledge Discovery at the Know-Center (competence centre for Big Data analytics and data-driven business). His research interest include Natural Language Processing, Machine Learning, with a focus on Data Science and Big Data Analytics. He applies these methods in fields like Scientific Publication Mining, Intelligent Transportation Systems, and Smart Production.
\end{IEEEbiography}

\end{document}